\documentclass[journal]{IEEEtran}

\usepackage{multirow}
\usepackage{amsmath}
\usepackage{xcolor}
\usepackage{soul}
\usepackage[utf8]{inputenc}
\usepackage[small]{caption}
\usepackage{amsfonts}
\usepackage{mathrsfs}
\usepackage{algorithm}
\usepackage{amsthm}
\usepackage{algpseudocode}
\usepackage{graphics}
\usepackage{epsfig}
\newtheorem{proposition}{Proposition}
\newcommand{\eg}{\emph{e.g. }}
\newcommand{\etal}{\emph{et.al.}}
\newcommand{\ie}{\emph{i.e. }}

\newcommand{\tabincell}[2]{\begin{tabular}{@{}#1@{}}#2\end{tabular}}

\begin{document}
	\title{Learning Student Networks via Feature Embedding}
	%
	%
	%
	
	\author{Hanting~Chen, Yunhe~Wang, Chang~Xu, Chao~Xu and Dacheng~Tao,~\IEEEmembership{Fellow,~IEEE}
		\thanks{Hanting Chen and Chao Xu are with the Key Laboratory of Machine Perception (Ministry of Education) and Coopertative Medianet Innovation Center, School of EECS, Peking University, Beijing 100871, P.R. China. E-mail:htchen@pku.edu.cn, xuchao@cis.pku.edu.cn}
		\thanks{Hangting Chen and Yunhe Wang are with the Noah's Ark Laboratory, Huawei Technologies Co., Ltd, HuaWei Building, No.3 Xinxi Road, Shang-Di Information Industri Base, Hai-Dian District, Beijing 100085, P.R. China. E-mail:htchen@pku.edu.cn, yunhe.wang@huawei.com}
		\thanks{Chang Xu and Dacheng Tao are with the UBTech Sydney Artificial Intelligence Centre and the School of Computer Science in the Faculty of Engineering and Information Technologies at The University of Sydney, J12 Cleveland St, Darlington NSW 2008, Australia. E-mail: c.xu@sydney.edu.au, dacheng.tao@sydney.edu.au.}}
	
	%
	%

	\markboth{SUBMITTED TO IEEE TRANSACTIONS XXX}%
	{Shell \MakeLowercase{\textit{et al.}}: Bare Demo of IEEEtran.cls for IEEE Journals}
	%

	\maketitle
	
	\begin{abstract}
		Deep convolutional neural networks have been widely used in numerous applications, but their demanding storage and computational resource requirements prevent their applications on mobile devices. Knowledge distillation aims to optimize a portable student network by taking the knowledge from a well-trained heavy teacher network. Traditional teacher-student based methods used to rely on additional fully-connected layers to bridge intermediate layers of teacher and student networks, which brings in a large number of auxiliary parameters. In contrast, this paper aims to propagate information from teacher to student without introducing new variables which need to be optimized. We regard the teacher-student paradigm from a new perspective of feature embedding. By introducing the locality preserving loss, the student network is encouraged to generate the low-dimensional features which could inherit intrinsic properties of their corresponding high-dimensional features from teacher network. The resulting portable network thus can naturally maintain the performance as that of the teacher network. Theoretical analysis is provided to justify the lower computation complexity of the proposed method. Experiments on benchmark datasets and well-trained networks suggest that the proposed algorithm is superior to state-of-the-art teacher-student learning methods in terms of computational and storage complexity.
	\end{abstract}

	\begin{IEEEkeywords}
		deep learning, teacher-student learning, knowledge distillation.
	\end{IEEEkeywords}

	\IEEEpeerreviewmaketitle

	\begin{figure*} 
		\centering
		\includegraphics[width=0.95\linewidth]{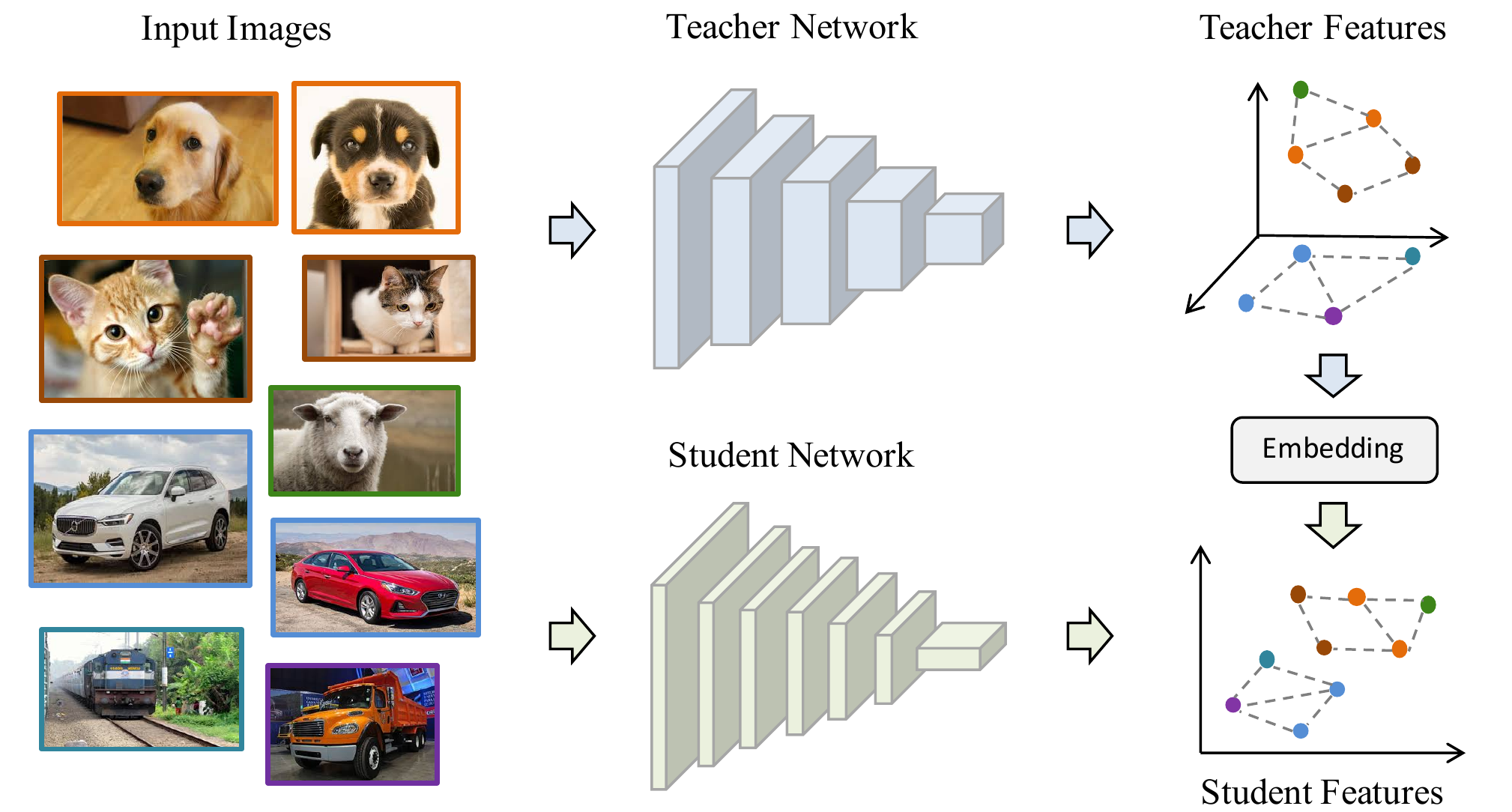}
		\caption{The diagram of the proposed method for learning portable deep neural networks. The top line is the teacher network and the bottom line is the student network. By applying the proposed method with a locality preserving loss, the student network embeds features in a low-dimensional space and maintains the relationship between samples in the high-dimensional space.}
		\label{fig1}
	\end{figure*}
	
	\section{Introduction}
	
	\IEEEPARstart{D}{eep} neural networks (DNNs)  have provided state-of-the-art performance in various fields such as image classification~\cite{VGG,krizhevsky2012imagenet}, semantic modeling~\cite{greff2017lstm,yuan2015scene}, visual quality evaluation~\cite{hou2015blind}, object detection~\cite{ren2015faster,gong2016change}, and segmentation~\cite{long2015fully}.  However, a neural network with considerable parameters requires heavy computation for both training and test, which is difficult to use on edge devices such as mobile phones and smart cameras. For example, a VGGNet~\cite{VGG} consisting of 16 convolutional layers has more than 500\emph{MB} parameters and it requires more than $10^{10}\times$ floating number multiplications, which cannot be tolerated by portable devices. Therefore, how to compress and accelerate existing CNNs has become a research hotspot.
	
	Recently, a variety of CNN compression methods have been proposed to tackle the aforementioned issues such as quantization~\cite{kmeanscompress,cheng2017quantized}, weight and feature approximation~\cite{SVD}, encoding~\cite{Hash}, approximation~\cite{gysel2018ristretto}, and pruning~\cite{wang2016cnnpack,han2015deep}. Wherein, weight pruning based methods achieve the highest compression performance since there are considerable subtle weights in most of pre-trained CNNs. In specific, Han~\etal~\cite{pruning15} showed that over 70\% subtle weights in AlexNet~\cite{krizhevsky2012imagenet} can be removed without affecting its original top-5 accuracy. Wang~\etal~\cite{wang2016cnnpack} further pointed out that the redundancy can exist in both large and small weights and considerable redundancy also exists in modern CNNs such as ResNet~\cite{he2016deep}.

	Although pruning based methods can provide very high compression and speed-up ratios, compressed CNNs by exploiting these approaches cannot be directly used in mainstream platforms (\eg Tensorflow and Caffe) and hardwares (\eg NVIDIA GPU cards) since they require special architectures and implementation tricks (\eg sparse convolution and Huffman encoding). As deeper networks often have higher performance than that of shallow networks~\cite{DoDeep}, the teacher-student learning paradigm ~\cite{hinton2015distilling,RDL,romero2014fitnets,you2017learning,yim2017gift,paying,huang2017like,wangAAAI18} has emerged to learn portable networks (student network) of deeper architecture yet fewer parameters and convolution filters from the original network (teacher network). Compared with other approaches, portable networks generated by the teacher-student paradigm are much more flexible since they are exactly regular neural networks which do not need any additional supports for implementing online inference. 
	
	However, besides directly making outputs of teacher and student networks similar, most of existing methods cannot directly inherit teacher information in other layers to the student network. Since the student network has a thinner architecture than its teacher's, the feature dimensionality (the number of filters) in the student network is much less than that of its teacher. Therefore, a lot of works ~\cite{romero2014fitnets,you2017learning,yim2017gift,paying,huang2017like,wangAAAI18} proposed to use an intermediate fully-connected layer to connect hint and guided layers to approximately inherit the teacher information. In addition, the fully-connected layer brings in a large number of auxiliary parameters, which have larger space and computational complexities and cannot be applied on large-scale CNNs in practice. Taking Student 4 in Table~\ref{table2} as an example, the memory usage for storing the fully-connected layer is about 135\emph{MB}, which is much larger than those of the student network (9\emph{MB}) and the teacher network (35\emph{MB}). 
	
	Given high-dimensional features from teacher network and low-dimensional features from student networks, it is natural to regard the information propagation between these two networks as a feature embedding task. Therefore, we propose a manifold learning based method for learning portable CNNs. In summary, the proposed approach makes the following contributions:
	\begin{itemize}
		\setlength{\itemsep}{4pt}
		\setlength{\parsep}{0pt}
		\setlength{\parskip}{0pt}
		\item We propose to bridge teacher and student networks via a feature embedding task, so that student networks with fewer filters can generate low-dimensional features to preserve relationships between examples.
		\item We introduce the locality preserving loss into the teacher-student learning paradigm and provide theoretical analysis to justify the lower computation complexity of the proposed algorithm.
		\item Experiments on benchmarks demonstrate that the proposed method can efficiently learn portable networks with state-of-the-art performance.
	\end{itemize}
	This paper is organized as follows. Section~\ref{sec:related} investigates related works on network compression algorithms. Section~\ref{sec:method} proposes the student network learning method by feature embedding. Section~\ref{sec:ana} analyzes the computational and space complexity of the proposed method. Section~\ref{sec:experi} shows the experimental results of the proposed method on several benchmark datasets and Section~\ref{sec:conclu} concludes this paper.

	\section{Related Works}\label{sec:related}
	
	Since CNNs require heavy computation and storage, it is difficult to adapt a neural network to real-world applications directly. Recently, various related works have been proposed to reduce the complexity of CNNs. Compression methods can be grouped into network trimming, layer decomposition and knowledge distillation based on their techniques.
	
	\subsection{Network Trimming}
	Network Trimming aims to remove redundant neurons in CNNs to accelerate and compress the original network. Gong~\etal~\cite{gong2014compressing} suggested vector quantization to represent similar connections using a cluster center. Denton~\etal~\cite{SVD} exploited the singular value decomposition approach and decomposed the weight matrices of fully connect layers. Considering that 32-bit floating numbers are overfined for parameters of CNNs, Courbariaux~\etal~\cite{courbariaux2016binarized} and Rastegari~\etal~\cite{rastegari2016xnor} explored binarized neural networks, whose weights are -1/1 or -1/0/1. Han~\etal~\cite{han2015deep} utilized pruning, quantization and Huffman coding to achieve a higher compression ratio. In addition, Wang~\etal~\cite{wang2016cnnpack} introduced the discrete cosine transform (DCT) bases and converted convolution filters into the frequency domain, thereby producing a much higher compression ratio and speed improvement. Subsequently, Wang~\etal~\cite{wang2017beyond} compressed feature map of CNN in the frequency domain and accelerated the calculation of convolution directly. Sun~\etal~\cite{sun2017design} introduced least absolute shrinkage and selection operator to design a selection method for efficient networks. Wang~\etal~\cite{wang2018novel} proposed a novel pruning algorithm with better generalization and pruning efficiency by utilizing Group Lasso. Huang and Yu~\cite{huang2018ltnn} compressed the weight matrices during training by reshaping them into a high-dimensional tensor with a low-rank approximation.

	Although aforementioned algorithms achieve satisfactory results in CNN compression, architectures of networks compressed by these methods would be significantly different from the that of ordinary networks, which means that special implementations are required for high-speed inference and development costs are increased.
	
	\subsection{Layer Decomposition}
	Traditional layers in DNNs (\eg convolutional layers, fully-connect layers) often result in huge computational cost. Therefore, a number of works aim to design lightweight layers to obtain efficient networks. Jin~\etal~\cite{jin2014flattened} presented fully factorized convolutions to accelerate the feedforward of deep networks. Wang~\etal~\cite{wang2017factorized} factorized the convolutional
	layer by considering spatial convolution. SqueezeNet~\cite{iandola2016squeezenet} achieved an accuracy similar with AlexNet yet with 50$\times$ fewer parameters by utilizing a bottleneck architecture. Pang~\etal~\cite{pang2018convolution} proposed sparse shallow MLP to construct a deep network with few parameters and high accuracy. Since the convolution calculation in CNNs is time-consuming, a various of algorithms focused on redesigning the convolutional layers. MobileNets~\cite{howard2017mobilenets} introduced depth-wise separable convolutions that largely reduced the computation cost of convolutional layers. ShuffleNet~\cite{zhang2017shufflenet} combined pointwise group convolution and channel shuffle to decrease the computational complexity while maintaining accuracy. Wu~\etal~\cite{wu2017shift} presented a parameter-free ``shift" operation to alternate vanilla convolutions. Moreover, Sandler~\etal~\cite{sandler2018mobilenetv2} improved MobileNets by introducing an inverted residual structure, which consists of thin and linear bottleneck layers. Ma~\etal~\cite{ma2018shufflenet} proposed a new metric beyond FLOPs to evaluate the speed of CNNs, which leads to ShuffleNet V2. Unlike the methods to design a low-cost convolutional layer, Wang~\etal~\cite{wang2018learning} reused the filters in CNNs by exploiting versatile filters and achieved a better performance.

	\subsection{Knowledge distillation}
	Different from directly compressing the heavy networks, some works investigate the intrinsic information of original networks to learn smaller networks. Knowledge Transfer, first pioneered by Hinton~\etal~\cite{hinton2015distilling}, aims to improve the training of a student network by borrowing knowledge  from another powerful teacher network, while the student network has fewer parameters. It uses a softened version of the final output of a teacher network called softened target to instruct the student network. Besides the outputs, features of intermediate layers in teacher networks also contain useful information which can guide the learning of student networks. Therefore, Romero~\etal~\cite{romero2014fitnets} minimized the difference between features of a hint layer in the student network and a guide layer in its teacher network, which enables the student network to receive sufficient information from teacher network. McClure and Kriegeskorte~\cite{RDL} proposed the pairwise distance of samples as a useful knowledge to transfer, which increases the robustness of student network. Motivated by ensemble learning methods, You~\etal~\cite{you2017learning} simultaneously utilized multiple teacher networks to learn a better student network. Moreover, several algorithms have been developed to investigate the restriction between teacher and student. Zagoruyk and Komodakis~\cite{paying} exploited attention mechanism and proposed to transfer attention maps that are the summaries of full activations. Huang and Wang~\cite{huang2017like} treat the knowledge transfer as a distribution matching problem and used the Maximum Mean Discrepancy (MMD) metric to minimize the difference between features from teacher and student networks. Wang~\etal~\cite{wangAAAI18} exploited generative adversarial network to make feature distribution of teacher and student networks similar.
	
	Compared to network trimming approaches, teacher-student learning paradigms can be applied to mainstream hardware without special requirements, and can be combined with layer decomposition methods easily. Existing knowledge distillation algorithms can learn efficient networks under the guidance of the teacher networks. However, these methods usually exploited a fully-connect layer to bridge the gap between the high-dimensional features from teacher network and low-dimensional features from student network, which introduced a large number of additional parameters. For example, a $19GB$ fully-connect layer is required to connect $7\times7\times2048$ dimensional intermediate features in the teacher network (\eg ResNet-101) and $7\times7\times1024$ dimensional intermediate features in the student network (\eg Inception-BN). Given these huge space and computational complexities, a flexible and effective algorithm to transfer knowledge between features of the teacher and student is urgently required.
	
	\section{Student Network Embedding}\label{sec:method}
	This section first reviews related works on the teacher-student learning paradigm, and then presents the student network embedding approach by introducing a locality preserving loss.
	
	\subsection{Teacher-Student Interactions}
	To learning student networks with portable architectures, Hinton~\etal~\cite{hinton2015distilling} first proposed the Knowledge Distillation (KD) approach, which utilizes a softened output of a teacher network to transform information to a smaller network. McClure and Kriegeskorte~\cite{RDL} further proposed to minimize the pairwise distance of samples after employing the student network and the teacher network. 
	
	Let $\mathcal{N}_{T}$ and $\mathcal{N}_{S}$ denote the original pre-trained convolutional neural network (teacher network) and the desired portable network (student network), respectively. The goal of knowledge distillation is utilizing $\mathcal{N}_{T}$ to enhance the performance of $\mathcal{N}_{S}$. Denote the softmax output of the teacher network  $\mathcal{N}_{T}$ as $P_{T}=\mbox{softmax}(a_{T})$, where $a_{T}$ denotes activations input into the softmax layer. Similarly, $P_{S}=\mbox{softmax}(a_{S})$ and $a_{S}$ are softmax output and activations of the student network $\mathcal{N}_{S}$, respectively. Hinton~\etal~\cite{hinton2015distilling} introduced the soften softmax output $\tau(P_{T})$ and $\tau(P_{S})$, which can be calculated as:
	\begin{equation}
	\tau(P_{T})=\mbox{softmax}(\frac{a_{T}}{\tau}), \quad \tau(P_{S})=\mbox{softmax}(\frac{a_{S}}{\tau}).
	\end{equation}
	Comparing with original one-hot like output, the soften output can transfer more information since it contains relationship between different classes. By matching the soften outputs of $\mathcal{N}_{T}$ and $\mathcal{N}_{S}$, student networks could inherits useful information from the teacher network. 
	The student network is then learned using the following loss function:
	\begin{equation}
	\label{eq1}
	\mathcal{L}_{KD}(\mathcal{N}_{S})=\mathcal{H}(\mbox{y},P_{S})+\lambda\mathcal{H}(\tau(P_{T}),\tau(P_{S})),
	\end{equation}
	where $\mathcal{H}$ is the cross-entropy loss, $\mbox{y}$ is the ground-truth label and $\lambda$ is a Trade-off parameter. The first term denotes the classical classification objective

	However, since architectures of teacher and student networks are significantly different, the constraint on the final output layer cannot be easily achieved. In addition, given the fact that features in the hidden layer also contain useful information, Romero~\etal~\cite{romero2014fitnets} presented a more flexible FitNet approach by introducing an intermediate hidden layer to connect teacher and student networks, which achieves higher performance than that of KD method. FitNet is trained in a two-stage fashion following the student-teacher paradigm. Specifically, a fully-connected layer is first added after the guided layer in the student network. Let $f_{S}$ and $f_{T}$ denote the features generated by guided layer and hint layer in the teacher network. The loss function used in the first stage can be formulated as:
	\begin{equation}
	\mathcal{L}_{HT}(\mathcal{N}_{S})=\frac{1}{2}\Vert r(f_{S})-f_{T}  \Vert^2,
	\label{Fcn:FitNet}
	\end{equation}
	where $r$ is the fully connect layer added to match $f_{S}$ and $f_{T}$. Then, the student network $\mathcal{N}_S$ is further tuned using knowledge distillation as illustrated in Eq.~\ref{eq1} in the second stage. Since the feature dimensionality of the intermediate hint layer is much higher than that of the softmax layer, FitNet can transfer more useful information thus yield a student network with higher performance.
	
	In addition, lots of works have been proposed to further enhance the accuracy of the student network by introducing different assumptions. For example, Yim~\etal~\cite{yim2017gift} introduced the FSP (Flow of Solution Procedure) matrix to transfer the relationship between convolutional layers. You~\etal~\cite{you2017learning} simultaneously utilized multiple teacher networks for learning a more accurate student network. Zagoruyko~\etal~\cite{paying} transferred useful information from the teacher network using the attention maps. Huang~\etal~\cite{huang2017like} minimized the Maximum Mean Discrepancy (MMD) loss between feature maps from teacher and student networks. Wang~\etal~\cite{wangAAAI18} utilized the generative adversarial network to make feature distributions of both teacher and student networks similar. However, there are still two important issues to be addressed: 1) Eq~\ref{Fcn:FitNet} independently considers each individual data point without investigating their connections; 2) the introduced fully-connected layer $r$ increases the cost for training the student network.
	
	\subsection{Locality Preserving Loss}
	As mentioned above, the feature dimensionality of the student network is lower than that of the teacher network (\eg from 6912 to 5120), thus we propose to regard the portable network learning as a low-dimensional embedding procedure, which aims to learn effective low-dimensional features. Considering that input images with similar content should lie on the neighbor area in both high-dimensional and low-dimensional spaces, we propose to exploit the manifold learning approach to address the teacher-student learning paradigm.
	
	Lots of nonlinear manifold learning methods have been proposed for obtaining accurate low-dimensional representations. For instance, locally linear embedding (LLE~\cite{roweis2000nonlinear}) attempts to represent the manifold locally by reconstructing each input point as a weighted combination of its neighbors. Isomap~\cite{tenenbaum2000global} preserves geometric distances by returning an embedding where the distances between points is approximately equal to the shortest path distance, and laplacian eigenmaps (LE~\cite{belkin2002laplacian}) builds a graph incorporating neighborhood information of the dataset to compute a low-dimensional representation of the data set that optimally preserves local neighborhood information in a certain sense. However, these nonlinear methods are not applicable  for large scale problems, due to their enormous computation and storage resource costs. In contrast, locality preserving projections (LPP~\cite{he2004locality}) is a linear alternative to those nonlinear methods and can be easily embedded into the learning of convolutional neural networks.
	
	Specifically, given a labeled training set with $n$ samples. $\{(\mbox{x}^1,\mbox{y}^1),(\mbox{x}^2,\mbox{y}^2),\cdots,(\mbox{x}^n,\mbox{y}^n)\}$, we denote features of $\mbox{x}^i$ extracted by the teacher and student networks as $f^i_{T}$ and $f^i_{S}$, respectively. Therefore, we propose to preserve the local relationship of the features generated by the student network like that of its teacher, which can be formulated as:
	\begin{equation}
	\min\limits_{W_{S}} \sum\limits_{i,j} \alpha_{ij} \Vert f_{S}^i - f_{S}^j\Vert_2^2,
	\label{Fcn:LP1}
	\end{equation}
	where $W_{S}$ is the parameters of the student network before the guided layer and $\alpha_{i,j}$ describes the local relationship between the features generated by the selected hint layer of the teacher network. Specifically, $\alpha_{i,j}$ is defined as follows:
	\begin{equation}
	\alpha_{i,j} =   
	\left\{
	\begin{aligned}
	& \mbox{exp}\left(-\frac{\Vert f_{T}^i - f_{T}^j \Vert^2_2}{\sigma ^2}\right) \;  &if\; {j\in N(i)},\\
	& 0 \;  &otherwise,
	\end{aligned} 
	\right.
	\end{equation}
	where $N(i)$ denotes the $k$ nearest neighbor of the feature $f_{T}^i$ of the $i$-th image $\mbox{x}^i$ generated by teacher network, and $\sigma$ is a normalized constant.
	
	\begin{algorithm}[t] 
		\caption{Learning student network by exploiting the proposed locality preserving loss.} 
		\label{alg1} 
		\begin{algorithmic}[1] 
			\Require 
			A given teacher network $\mathcal{N}_T$ and its training set $\mathcal{X}$ with $n$ instances, and the corresponding $k$-label set $\mathcal{Y}$, parameters: $\lambda$, $\gamma$, and $\tau$. 
			\State Initialize the student network $\mathcal{N}_S$, whose number of parameters is significantly fewer than that in $\mathcal{N}_T$;
			\Repeat
			\State Randomly select a batch $\{(\mbox{x}^i,\mbox{y}^i)\}^m_{i=1}$;
			\State Employ the teacher network on the mini-batch: \\\quad\quad\quad\quad$[\tau(P_{T}),P_T,f_T]\leftarrow \mathcal{N}_T(\mbox{x})$;
			\State Employ the student network on the mini-batch: \\\quad\quad\quad\quad $[\tau(P_{S}),P_S,f_S]\leftarrow \mathcal{N}_S(\mbox{x})$;
			\State Calculate the LP loss $\mathcal{L}_{LP} \leftarrow \frac{1}{m} \sum\limits_{i,j} \alpha_{ij} \Vert f_{S}^i - f_{S}^j\vert_2^2$;
			\State Calculate the loss function $\mathcal{L}_{Total}$ (Fcn.\ref{eq3});
			\State Update weights in $\mathcal{N}_S$ using gradient descent; 
			\Until convergence
			\Ensure 
			The student network $\mathcal{N}_S$.
		\end{algorithmic} 
	\end{algorithm}

	We can obtain a student network $\mathcal{N}_S$ preserving relationship between samples from a high-dimensional space to the target low-dimensional space by optimizing Eq.~\ref{Fcn:LP1}. However, to compute the $k$ nearest neighbors, we need to take the entire training set in each iteration, which is inefficient. Therefore, we use the mini-batch strategy to learn the student network, and the $k$-nearest neighbors will only be discovered within the mini-batch. The locality preserving loss function is therefore reformulated as:
	\begin{equation}
	\mathcal{L}_{LP} = \frac{1}{2m} \sum\limits_{i,j} \alpha_{ij} \Vert f_{S}^i - f_{S}^j\vert_2^2,
	\label{Fcn:LP2}
	\end{equation}
	where $m$ is the batch size of the student network.
	
	In addition, the ground-truth label data is also used for helping the training process of the student network. The entire objective function of the proposed network is then formulated as: 
	\begin{equation}
	\small
	\label{eq2}
	\mathcal{L}_{Total} = \frac{1}{2m}\left[\sum\limits_{i}\mathcal{H}(\mbox{y}^{i},P^i_{S}) +\gamma\sum\limits_{i,j} \alpha_{ij} \Vert f_{S}^i - f_{S}^j\Vert_2^2\right],
	\end{equation}
	where $\gamma$ is the weight parameter for seeking the trade-off of two different terms and $P^i_\mathcal{S}$ is the output of the classifier in the student network for the $i$-th sample $\mbox{x}^i$.
	
	The first term in Fcn. \ref{eq2} minimizes the cross entropy loss of classifier outputs to maintain the performance of the student network, and the second term embeds samples from the high-dimensional space to the low-dimensional space in the portable student network $\mathcal{N}_S$. Nevertheless, the knowledge distillation approach as discussed in Eq.~\ref{eq1} can be incorporated to further inherit more useful information from the teacher network. Therefore, we reformulate Eq.~\ref{eq2} as
	\begin{equation}
	\label{eq3}
	\begin{aligned}
	\mathcal{L}_{Total} = &\frac{1}{m}\sum\limits_{i}\left[\mathcal{H}(\mbox{y}^{i},P^i_{T})
	+\lambda\mathcal{H}(\tau(P^i_{T}),\tau(P^i_{T}))\right]\\
	&+\gamma\frac{1}{m}\sum\limits_{i,j} \alpha_{ij} \Vert f_{S}^i - f_{S}^j\Vert_2^2 .
	\end{aligned}
	\end{equation}
	
	Then, we use stochastic gradient descent (SGD) approach to optimize the student network. Since the proposed LP loss is exactly a linear operation, and the gradient of $\mathcal{L}_{LP}$ with respect to $f_{S}^i$ can be easily calculated as:
	\begin{equation}
	\frac{\mathcal{L}_{LP}}{\partial f_{S}^i} = \frac{1}{m}\sum\limits_{j:j\neq i} \alpha_{ij} (f_{S}^i - f_{S}^j ).
	\end{equation}
	The first term in Eq.~\ref{eq3}, \ie the classification loss, will affect all parameters in the student network, and parameters in $\mathcal{N}_S$ before the guided layer will be additionally updated by:
	\begin{equation}
	\frac{\partial \mathcal{L}_{LP}}{\partial W_{S}} = \sum\limits_{i=1}^m\frac{\partial \mathcal{L}_{LP}}{\partial f_{S}^i} \cdot\frac{\partial f_{S}^i}{\partial W_{S}},
	\end{equation}
	where $\frac{\partial f_{S}^i}{\partial W_{S}}$ is the gradient of the feature $ f_{S}^i$. Alg.\ref{alg1} summarizes the detailed procedure of the proposed approach for learning student networks.

	\section{Analysis on the Complexity}\label{sec:ana}
	As mentioned above, traditional FitNet based methods introduce an additional fully-connected layer, which makes the training procedure of the student network very slow. In contrast, the proposed method does not need the fully-connected layer for connecting teacher and student networks. Thus, the space complexity and the computational complexity are much lower than those of FitNet based methods, as analyzed in Proposition~\ref{prop}.
	
	\begin{proposition}
		Denote dimensionalities of features generated by the teacher network and the student network as $d_T$ and $d_S$, respectively. For a mini-batch with $m$ samples, the computational complexity of the proposed scheme for inheriting information from the teacher network to the student network is $\mathcal{O}(m^2(d_S+d_T))$.
		\label{prop}
	\end{proposition}
	
	\begin{proof}
		The proposed method inherits information from the teacher network to the student network by calculating the locality preserving loss $\mathcal{L}_{LP}$, whose computational complexity is calculated through three steps. 
		
		The first step is to calculate distances between features generated by the teacher network, \ie $\Vert f_{T}^i - f_{T}^j\Vert_2^2$, whose computational complexity is $\mathcal{O}(m^2d_T)$. The second step is to find the $k$ nearest neighbors of each feature $f_{T}^i$ to calculate $\alpha_{ij}$ in Eq.~\ref{Fcn:LP2}, whose computational complexity is $\mathcal{O}(km^2)$. Then, distances between features generated by the student network, \ie $\Vert f_{S}^i - f_{S}^j\Vert_2^2$ will be calculated to obtain $\mathcal{L}_{LP}$, with the computational complexity $\mathcal{O}(m^2d_S)$. Thus, the computational complexity of the proposed method is $\mathcal{O}(m^2(d_S+d_T+k))$.
		
		Considering that, the dimensionality of features in a CNN is usually much larger than $k$ and $m$, \eg $k = 5$, $m=128$, $d_T = 6\times6\times192=6912$, and $d_S=8\times8\times80=5120$ in the Student 4 network as illustrated in Table~\ref{table2}, the complexity $\mathcal{O}(km^2)$ could be ignored. Therefore, the computational complexity of the proposed method is $\mathcal{O}(m^2(d_S+d_T))$.
	\end{proof}
	
	In contrast, traditional methods~\cite{romero2014fitnets,wangAAAI18,you2017learning,paying,huang2017like} use a fully-connected layer to build the connection between teacher and student network have a $\mathcal{O}(md_Sd_T)$ computational complexity. According to Proposition~\ref{prop}, the proposed teacher-student learning paradigm has a much smaller computational complexity, which will accelerate the learning process for portable student networks. Taken the Student 4 network in Table~\ref{table2} as an example, $md_Sd_T/m^2(d_S+d_T)\approx23$, since $m$ is much smaller than either $d_S$ or $d_T$. In addition, considering $k = 5$, $m=256$, $d_T = 7\times7\times1024=50176$, and $d_S=7\times7\times2048=100352$ in Table~\ref{table4}, $md_Sd_T/m^2(d_S+d_T)\approx131$. If the teacher network becomes more complex, $d_T$ and $\mathcal{O}(md_Sd_T)$ will be increased significantly. In addition, the fully-connected layer introduces considerable parameters with a $\mathcal{O}(d_Sd_T)$ space complexity, to store such a fully-connected layer would require more than 100\emph{MB} memory usage in practice. While the memory usage for storing parameter in the student network is only about 9\emph{MB}. In contrast, the proposed method does not have any additional parameters. We will further illustrate this superiority in the experiment part.
	
	\begin{figure*}[t]
		\centering
		\begin{tabular}{cc}
			\includegraphics[width=0.3\linewidth]{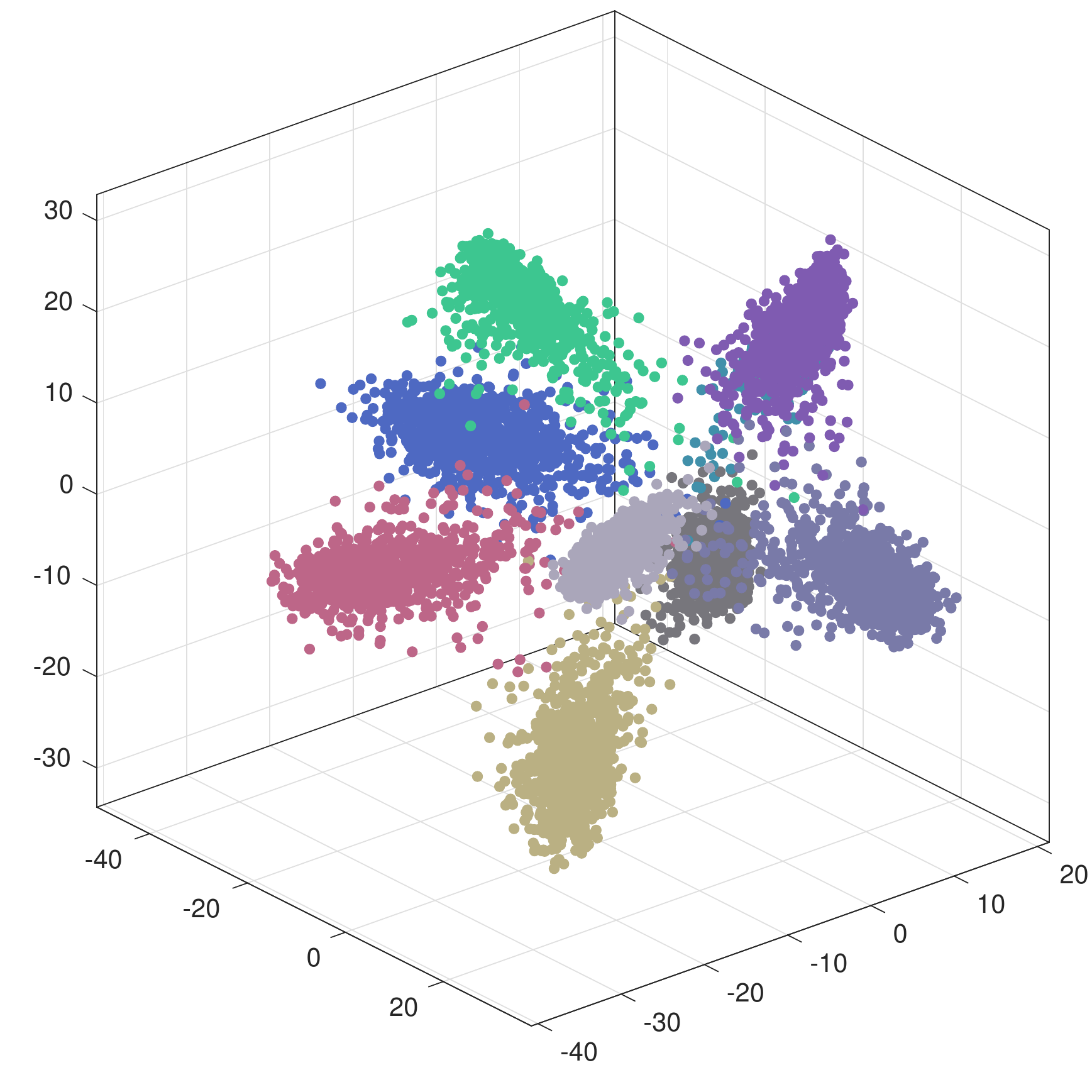} \quad\quad&\quad\quad
			\includegraphics[width=0.3\linewidth]{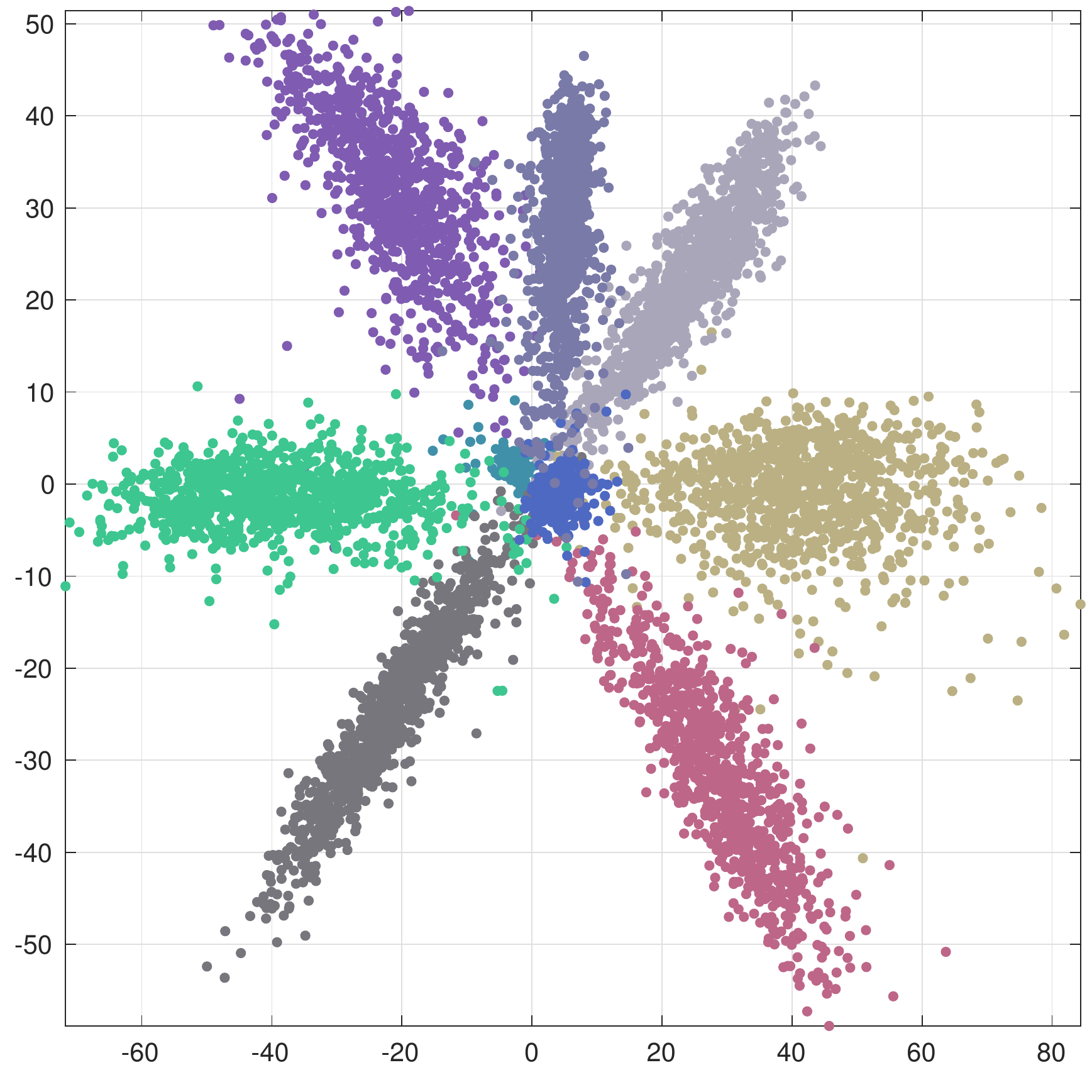} \\
			(a) accuracy = $99.3\%$ \quad\quad&\quad \quad(b) accuracy = $97.5\%$  \\&\\
			\includegraphics[width=0.3\linewidth]{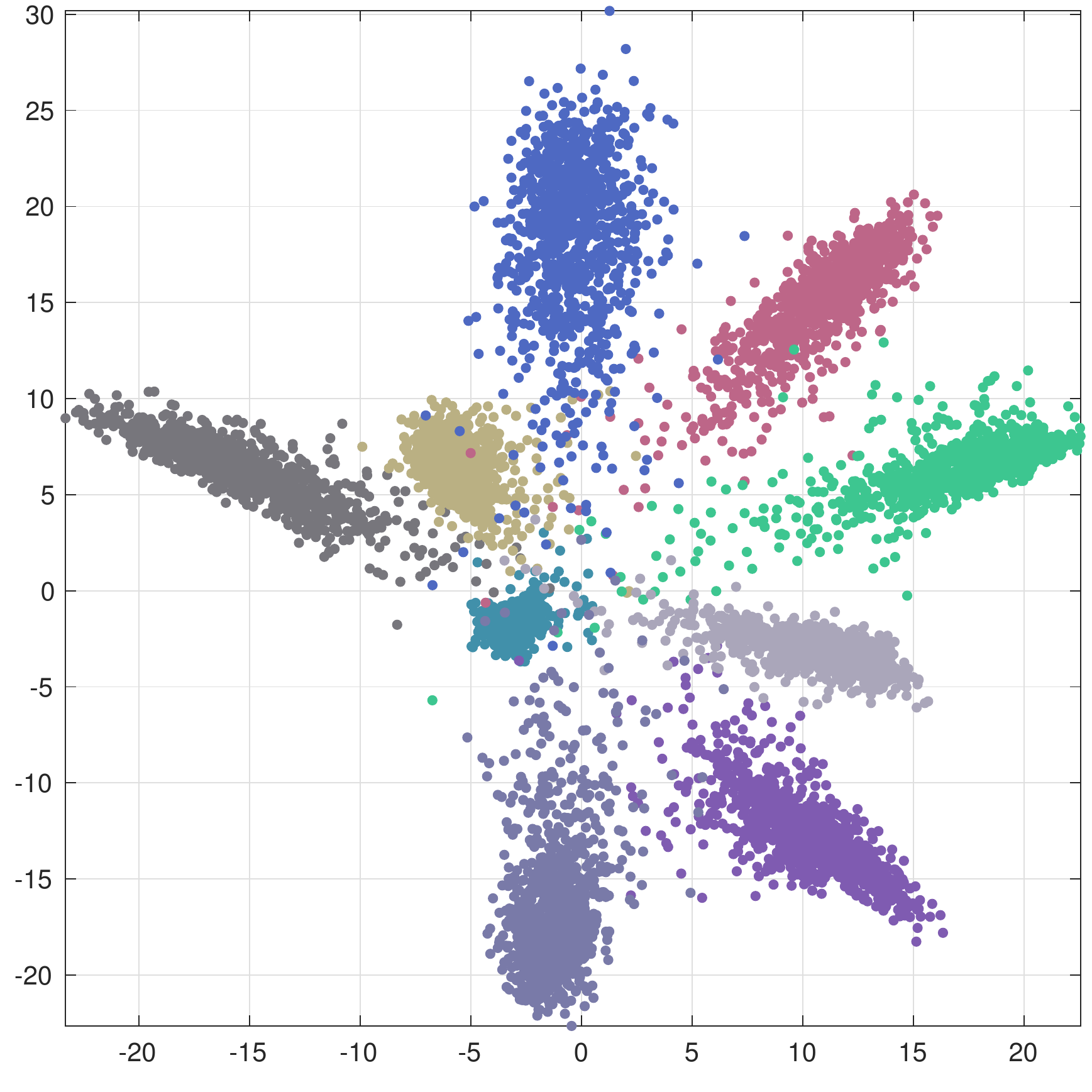}\quad\quad &\quad\quad
			\includegraphics[width=0.3\linewidth]{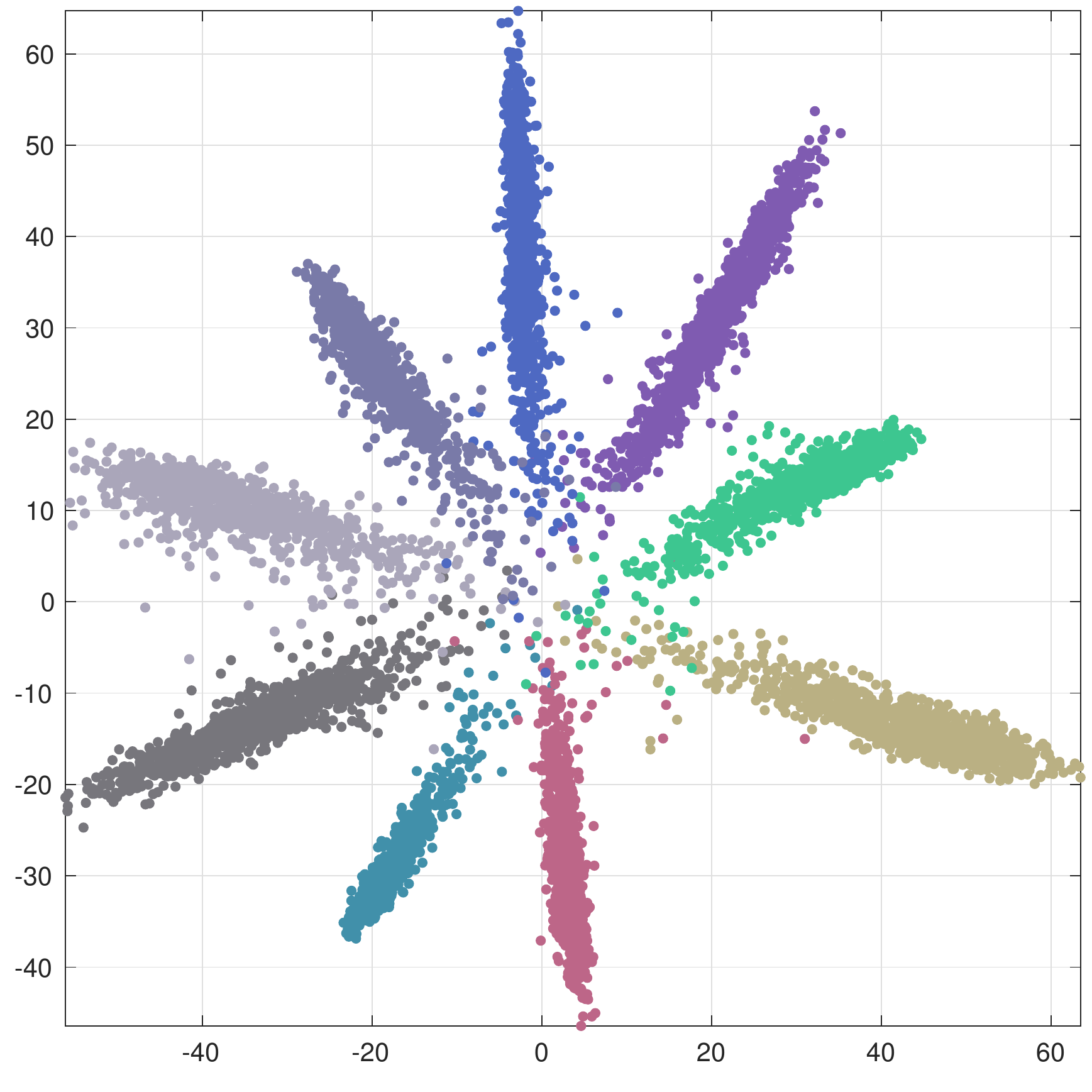} \\
			(c) accuracy = $98.5\%$\quad\quad& \quad\quad(d) accuracy = $99.2\%$\\
		\end{tabular}
		\caption{Visualization of features generated by different networks on the MNIST test set. (a) features of the original teacher network; (b) features of the student network learned using the standard back-propagation strategy; (c) features of the student network learned using FitNet; (d) features of the student network learned using the proposed method with locality preserving loss. Note that features of the same category in each figures are marked with the same color.}
		\label{Fig:visual}
	\end{figure*}

	\section{Experiments}\label{sec:experi}
	In this section, we implement experiments to validate the effectiveness of the method on three benchmark datasets, including MNIST, CIFAR-10, and CIFAR-100. Experimental results are further analyzed and discussed to investigate the benefits of the proposed method.

	\subsection{Validations on MNIST}
	\noindent\textbf{Visualization of Features.}
	The locality preserving loss was introduced in Fcn.~\ref{eq3}, which aims to learn the effective low-dimensional features from the pre-trained teacher network. In order to illustrate the superiority of the proposed method, we first trained a LeNet++~\cite{centerloss} as the teacher network, which has six convolutional layers and a fully-connected layers for extracting powerful 3D features. Numbers of neurons in each convolutional layer are 32, 32, 64, 64, 128, 128, and 3, respectively. This teacher network achieved a $99.3\%$ test accuracy on the MNIST dataset, and the memory usage for storing all convolution filters of this teacher network is about $2,982$\emph{KB}.
	
	Then, a thinner student network with also six convolutional layers but half convolution filters per layer was initialized. As for the fully-connected layer, the number of neurons was set as 2 for seeking a low-dimensional embedding. Then, we trained this network using conventional back-propagation scheme, and the FitNet method on the MNIST dataset, respectively. Since the student network has fewer parameters, the accuracy of the network using conventional BP is only $97.5\%$, and the accuracy of the network using FitNet is about $98.5\%$.
	
	We then trained a new student network with the same architecture by exploiting the proposed method as described in Fcn.~\ref{eq3}. $\lambda$ and $\tau$ were equal to $2$ and $0.5$, respectively, which refer to those in Hinton~\etal~\cite{hinton2015distilling}. The learning rate $\eta$ was set to be $0.01$, empirically. The hyper-parameter $k$ for searching neighbors was set as 5, and $\gamma$ was set to be $1$. The accuracy of the resulting student network is $99.2\%$ which is slightly lower than that of its teacher network but much higher than that of the student network straightforwardly learned using the conventional back-propagation method. In addition, the memory usage for convolution filters of this student network is about $734$\emph{KB}, which only accounts for $\frac{1}{4}$ of that of the original teacher network. 
	
	\begin{table*}[t]
		
		\caption{The performance of the proposed method on student networks with various architectures.}
		\label{table2}
		\small
		\addtolength{\tabcolsep}{-3pt}
		\begin{center}
			\begin{tabular}{|c|c|c|c|c|c|c|c|c|c|c|c|}
				\hline
				\multirow{2}{*}{\textbf{Network}} & \multirow{2}{*}{\textbf{\#layers}} &\multirow{2}{*}{\textbf{\#params}}&\multirow{2}{*}{\textbf{\#mult}}&\multirow{2}{*}{\textbf{speed-up}}&\multirow{2}{*}{\textbf{compression}}&\multicolumn{2}{|c|}{\textbf{Training time (min)}}&\multicolumn{2}{|c|}{\textbf{Additional params}}&\multicolumn{2}{|c|}{\textbf{Accuracy}}\\
				\cline{7-12}
				& & & & & & \textbf{FitNet}&\textbf{Ours} &\textbf{FitNet}&\textbf{Ours}&\textbf{FitNet}&\textbf{Ours}\\
				\hline
				\hline
				Teacher &5 &$\sim$9M& $\sim$725M & $\times1$ & $\times1$ &\multicolumn{2}{|c|}{150} &\multicolumn{2}{|c|}{-}& \multicolumn{2}{|c|}{90.21\%}\\ 
				\hline
				Student 1&11 &$\sim$250K& $\sim$30M & \textbf{$\times$13.17}& \textbf{$\times$36}&146 & 93&$\sim$14M &0& 89.03\%& 89.56\%\\
				\hline
				Student 2&11 &$\sim$862K& $\sim$108M & $\times4.56$& $\times10.44$& 221&147 &$\sim$35M &0& 91.01\%& 91.33\%\\
				\hline
				Student 3&13 &$\sim$1.6M& $\sim$392M & $\times1.40$& $\times5.62$& 317& 200 &$\sim$35M &0& 91.14\%& 91.57\%\\
				\hline
				Student 4&19 &$\sim$2.5M& $\sim$382M & $\times1.58$ & $\times3.60$& 400 & 265 &$\sim$35M &0& 91.55\%& \textbf{91.91\%}\\
				\hline
			\end{tabular}
		\end{center}
	\end{table*}

	\begin{table*}[t]
		\caption{Different architectures on the CIFAR-10 dataset.}
		\label{table:archi}
		\small
		\begin{center}
			\begin{tabular}{|c|c|c|c|c|}
				\hline
				\textbf{Teacher}&\textbf{Student 1}&\textbf{Student 2}&\textbf{Student 3}&\textbf{Student 4}\\
				\hline
				\hline
				\quad conv 3x3x96 \quad\quad&\quad conv 3x3x16\quad\quad&\quad conv 3x3x16\quad\quad&\quad conv 3x3x32\quad\quad&\quad conv 3x3x32\quad\quad\\
				pool 4x4&conv 3x3x16&conv 3x3x32&conv 3x3x48&conv 3x3x32\\
				&conv 3x3x16&conv 3x3x32&conv 3x3x64&conv 3x3x32\\
				&pool 2x2&pool 2x2&conv 3x3x64&conv 3x3x48\\
				&&&pool 2x2&conv 3x3x48\\
				&&&&pool 2x2\\
				\hline
				conv 3x3x96&conv 3x3x32&conv 3x3x48&conv 3x3x80&conv 3x3x80\\
				pool 4x4&conv 3x3x32&conv 3x3x64&conv 3x3x80&conv 3x3x80\\
				&conv 3x3x32&conv 3x3x80&conv 3x3x80&conv 3x3x80\\
				&pool 2x2&pool 2x2&conv 3x3x80&conv 3x3x80\\
				&&&pool 2x2&conv 3x3x80\\
				&&&&conv 3x3x80\\
				&&&&pool 2x2\\
				\hline
				conv 3x3x96&conv 3x3x48&conv 3x3x96&conv 3x3x128&conv 3x3x128\\
				pool 4x4&conv 3x3x48&conv 3x3x96&conv 3x3x128&conv 3x3x128\\
				&conv 3x3x64&conv 3x3x128&conv 3x3x128&conv 3x3x128\\
				&pool 8x8&pool 8x8&pool 8x8&conv 3x3x128\\
				&&&&conv 3x3x128\\
				&&&&conv 3x3x128\\
				&&&&pool 8x8\\
				\hline
				fc&fc&fc&fc&fc\\
				softmax&softmax&softmax&softmax&softmax\\
				\hline
			\end{tabular}
		\end{center}
	\end{table*}
	
	\begin{table}[h]
		\caption{Classification error on the MNIST dataset.}
		\label{table1}
		\vspace{-0.5em}
		\begin{center}
			\normalsize
			\begin{tabular}{|c|c|c|}
				\hline
				\textbf{Algorithm} & \textbf{\#params} &\textbf{Misclass}\\
				\hline
				\hline
				Teacher & $\sim$361K& 0.55\%\\
				\hline
				Standard back-propagation & $\sim$30K& 1.90\%\\
				\hline
				\tabincell{c}{Knowledge Distillation\\\small{~\cite{hinton2015distilling}}}& $\sim$30K &0.65\%\\
				\hline
				FitNet\small{~\cite{romero2014fitnets}} & $\sim$30K &0.51\%\\
				\hline
				Student (Ours) & $\sim$30K& \textbf{0.48\%}\\
				\hline
			\end{tabular}
		\end{center}
	\end{table}
	
	Moreover, features (\ie input data of the softmax layer) of the above three networks were visualized in Fig.~\ref{Fig:visual}. Fig.~\ref{Fig:visual} (a) shows that features of different categories extracted using the original teacher network are separated from each other in the 3-dimensional space, and can be easily distinguished by the following softmax layer. Since the student network has fewer parameters, features of the network trained using the conventional back-propagation are distorted as illustrated in Fig.~\ref{Fig:visual} (b). Therefore, the performance of this network is lower than that of the teacher network. Fig.~\ref{Fig:visual} (d) shows features of the student network learned using the proposed scheme with LP loss. It is clear that, features are separated by supervising of the proposed locality preserving loss, which can be seen as an excellent low-dimensional embedding of features learned by the teacher network.
	
	\begin{table*}[t]
		\caption{10-Class results of different networks on the CIFAR-10 datasets.}
		\label{table:10class}
		\begin{center}
			\normalsize
			\begin{tabular}{|c|c|c|c|c|c|c|c|c|c|c|}
				\hline
				\textbf{Algorithm} & \textbf{plane} &\textbf{car} &\textbf{bird} &\textbf{cat} &\textbf{deer} &\textbf{dog} &\textbf{frog} &\textbf{horse} &\textbf{ship} &\textbf{truck} \\
				\hline
				\hline
				Teacher &90.1\%&93.8\%&86.0\%&74.6\%&93.5\%&86.2\%&95.2\%&92.6\%&95.3\%&95.2\% \\ 
				\hline
				FitNet~\cite{romero2014fitnets}&90.7\%&97.6\%&91.0\%&82.7\%&93.8\%&86.2\%&92.7\%&93.6\%&94.6\%&93.5\%\\
				\hline
				Knowledge Distillation~\cite{hinton2015distilling}&90.0\%&95.2\%&83.2\%&84.4\%&93.2\%&87.1\%&95.0\%&91.6\%&97.3\%&93.7\%\\
				\hline
				Multiple Teachers~\cite{you2017learning}&91.0\%&96.8\%&90.0\%&83.1\%&92.8\%&87.1\%&93.3\%&94.2\%&94.4\%&93.8\%\\
				\hline
				FSP Learning~\cite{yim2017gift}&90.9\%&96.7\%&90.4\%&82.9\%&93.3\%&87.1\%&95.6\%&92.7\%&96.0\%&93.1\%\\
				\hline
				Adversarial Learning~\cite{wangAAAI18}&91.3\%&96.5\%&89.8\%&84.2\%&91.8\%&87.5\%&93.1\%&95.3\%&93.4\%&93.8\%\\
				\hline
				\hline
				Student(Ours)&91.1\%&97.0\%&90.2\%&83.6\%&92.4\%&87.2\%&95.3\%&93.2\%&95.1\%&94.1\%\\
				\hline
			\end{tabular}
		\end{center}
	\end{table*}
	
	\noindent\textbf{Compression Results.}
	In order to further illustrate the superiority of the proposed method, we followed the setting in Romero~\etal~\cite{romero2014fitnets} and Wang~\etal~\cite{wangAAAI18} to conduct the student network learning experiment on the MNIST dataset. The teacher network consists of maxout convolutional layers as reported in Goodfellow~\etal~\cite{goodfellow2013maxout} and the student network is twice as deep as the teacher network following Romero~\etal~\cite{romero2014fitnets}. The hyper-parameter $k$ of the proposed method for searching neighbors was set as 5, and $\gamma$ was set to be $1$, which were tuned on the last 5,000 images in the train set. The impact of parameters would be showed in the following experiment. The teacher network contains 3 maxout layers and a fully-connected layer, and the student network has 6 maxout layers and a fully-connected layer. Parameters in the student network is only about 8\% of that in the teacher network. The guide layer of the teacher network is the 2nd layer while the hint layer of the student network is the 4th layer, respectively. The parameters of the networks were initialized randomly in U(-0.005,0.005) and the networks were trained using stochastic gradient descent with RMSProp whose learning rate is 0.0005 and weight decay is 0.9. 
	
	Table \ref{table1} reports the results of different networks on the MNIST dataset by exploiting the proposed method. In order to illustrate the advantage of the introduced locality preserving loss, the performance of student networks with the same architecture trained by using standard back-propagation, knowledge distillation~\cite{hinton2015distilling}, and FitNets~\cite{romero2014fitnets} was also reported. It can be found in Table~\ref{table1}, the student network trained using the standard back-propagation achieved a 1.90\% error rate. The student network learned utilizing the knowledge distillation obtained a 0.65\% misclassification error. The error rate of the student network trained by exploiting the FitNet approach is 0.51\%, which outperforms both conventional back-propagation and knowledge distillation methods, and is slightly lower than that of the teacher network. In contrast, the student network using the proposed method achieves a 0.48\% accuracy.
	
	\begin{figure}[h]
		\centering
		\includegraphics[width=0.9\linewidth,trim={0.2cm 0cm -0.5cm 0cm}]{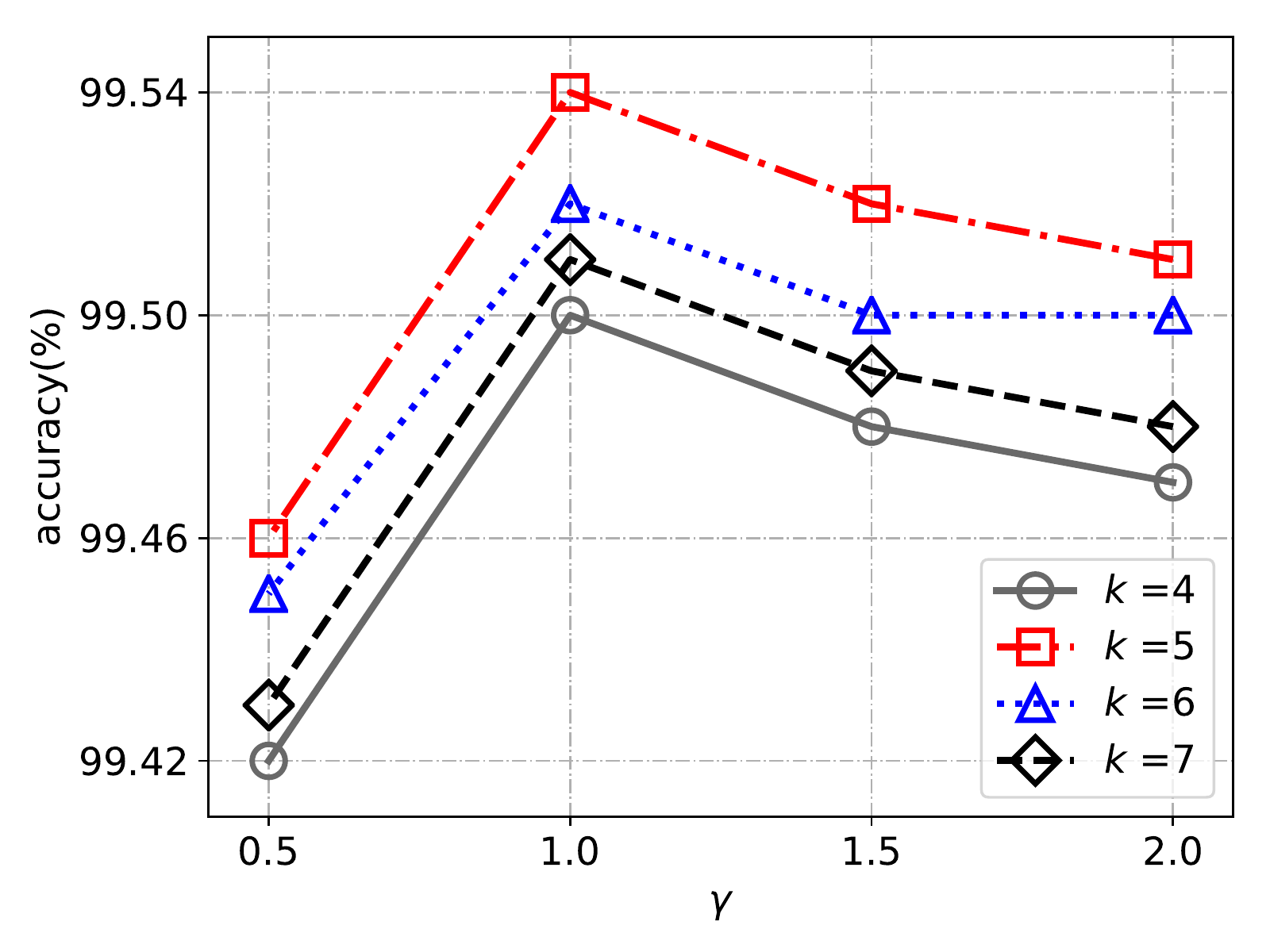}
		\caption{The performance of the proposed method with different parameters $\alpha$ and $\beta$ on the MNIST dataset.
		}
		\label{fig:lambda}
	\end{figure}
	
	\noindent\textbf{Impact of parameters.}
	As discussed above, the proposed method has two hyper-parameters: $k$ and $\gamma$. We tested their impact on the accuracy of the student network on the validation set of MNIST dataset. It can be seen from Figure~\ref{fig:lambda} that the student network trained utilizing the proposed method achieved the highest accuracy (99.54\%) on the validation set when $k=5$ and $\gamma=1$. Therefore, we kept these hyper-parameters for the proposed method on the following experiments.
	
	\begin{table*}[t]
		\caption{Classification results of different networks on CIFAR-10 and CIFAR-100 datasets.}
		\label{table3}
		\begin{center}
			\normalsize
			\begin{tabular}{|c|c|c|c|c|}
				\hline
				\textbf{Algorithm} & \textbf{\#layers} &\textbf{\#params}&\textbf{CIFAR-10}&\textbf{CIFAR-100}\\
				\hline
				\hline
				Teacher &5 &$\sim$9M & 90.21\% & 62.78\% \\ 
				\hline
				Student(Ours)&19 &$\sim$2.5M& \textbf{91.91\%}& \textbf{65.13\%}\\
				\hline
				FitNet~\cite{romero2014fitnets}&19 &$\sim$2.5M& 91.55\%& 64.89\%\\
				\hline
				Knowledge Distillation~\cite{hinton2015distilling}&19 &$\sim$2.5M& 91.04\%& 63.07\%\\
				\hline
				Multiple Teachers~\cite{you2017learning}&19 &$\sim$2.5M& 91.66\%& 65.06\%\\
				\hline
				FSP Learning~\cite{yim2017gift}&19 &$\sim$2.5M& 91.89\%& 64.65\%\\
				\hline
				Adversarial Learning~\cite{wangAAAI18}&19 &$\sim$2.5M& 91.68\%& 65.11\%\\
				\hline
			\end{tabular}
		\end{center}
	\end{table*}
	
	\subsection{Validations on CIFAR-10}
	After investigating the performance of the proposed method on the MNIST dataset, we then evaluated our method on the CIFAR-10 dataset, which consists of $32\times32$ pixel RGB color images with 10 categories. There are 50,000 training images and 10,000 testing images in this dataset. For fair comparison, images in the dataset were first processed using global contrast normalization (GCA) and ZCA whitening as suggested in Goodfellow~\etal~\cite{goodfellow2013maxout} and Romero~\etal~\cite{romero2014fitnets}, and the last 10,000 training images were selected as the validation set for tuning the hyper-parameters. We followed Goodfellow~\etal~\cite{goodfellow2013maxout} to train a teacher network, which consists of three convolutional layers and a fully-connected layer of 500 units with 5-linear-piece maxout activations. The teacher network has about 9M parameters with a 90.21\% classification accuracy, and trained for 400 epochs using RMSProp with an initial learning rate of 0.005 and a weight decay of 0.9. Note that the intermediate hidden layer of the teacher network and student network used in the proposed methods is same as that of FitNet~\cite{romero2014fitnets} for a fair comparison.
	
	\noindent\textbf{Trade off between compression/speed-up and accuracy.} Since the CIFAR-10 dataset consists of more complex images, its samples cannot be easily distinguished by a casually designed neural network. Therefore, several student networks with different architectures, as reported in Romero~\etal~\cite{romero2014fitnets}, were established to further explore the compression performance and accuracy of the proposed method. Table~\ref{table:archi} showed the detailed structures of these student networks. Numbers of parameters in these student networks are 250K, 862K, 1.6M, and 2.5M, respectively. The compression ratio and the speed-up ratio of each student network can be directly calculated by comparing its numbers of parameters and the floating number multiplications to those of the teacher network, respectively. Detailed results were reported in Table~\ref{table2}. 
	
	It can be found in Table~\ref{table2} that, the student network (Student 1) with the highest compression and speed-up ratios has a lower classification accuracy. While, when we appropriately increase the number of parameters, student networks (Student 2-4) can achieve higher performance than that of the teacher network. Moreover, compared with the results of student networks learned using FitNet~\cite{romero2014fitnets}, the proposed method achieves better performance of all the four student networks. 
	
	\noindent\textbf{Complexity.} As discussed in Proposition~\ref{prop}, the proposed method has significantly lower computational and space complexities. Therefore, we also report training time and additional parameters (\ie weights in the fully-connected layer) of the proposed method and conventional FitNet based method in Table~\ref{table2}. It is obvious that, the training time using the proposed method is much less than those of the FitNet and other FitNet based methohds, and we do not need any additional parameters for training student networks. Considering that there are lots of networks with heavy architectures~\cite{VGG,ren2015faster}, the proposed method is a much more flexible approach with higher performance and efficiency.
	
	\noindent\textbf{Comparison with state-of-the-art methods.} To illustrate the superiority of the proposed method, we compared it with other student-teacher learning methods using the same architecture (Student 4) on the CIFAR-10 dataset as summarized in Table~\ref{table3} and Table~\ref{table:10class}. As a result, the student network utilizing the proposed method achieved a 91.91\% accuracy. Comparison results show that our student network outperforms networks produced by state-of-the-art approaches, which proves that the proposed LP loss transfers more useful and intrinsic information from the teacher network. In addition, the student network with more portable architecture (Student 3) can achieve a 91.57\% accuracy, which is slightly higher than that of the student network (Student 4) trained using FitNet.
	
	\subsection{Validations on CIFAR-100}
	Moreover, we also conducted the validation on the CIFAR-100 dataset. This dataset has 60,000 RGB color images of pixel 32$\times$32, which is the same as that of CIFAR-10. However, the CIFAR-100 dataset has 100 categories with only 600 images per class, which is a more challenging benchmark dataset for conducting the classification experiment. For example, the teacher network used in Table~\ref{table3} on the CIFAR-100 dataset is only about 62\%, which is much lower than that on the CIFAR-10 dataset. Similarly, images in this dataset were pre-processed using global contrast normalization and ZCA whitening. Random flipping, random crop and zero padding was also used for data augmentation as suggested in Romero~\etal~\cite{romero2014fitnets}.
	
	We used the fourth student network (Student 4 in Table~\ref{table2}) to conduct the experiment on the CIFAR-100 dataset. Table~\ref{table3} reports the classification results of student networks learned by exploiting different student-teacher learning paradigms with the same architecture. As a result, the student network learned by the proposed method obtained a 65.13\% accuracy. It is clear that the student network learned using the proposed LP loss outperforms those of networks generated by the state-of-the-art methods, which demonstrated that the proposed method inherits useful information from the teacher network in a more effective and efficient way.

	\begin{table*}[t]
		\caption{Classification results of different networks on ImageNet datasets.}
		\label{table4}
		\begin{center}
			\normalsize
			\begin{tabular}{|c|c|c|c|c|}
				\hline
				\textbf{Algorithm} & \textbf{Model} &\textbf{\#params}& \textbf{Top-1}&\textbf{Top-5}\\
				\hline
				\hline
				Teacher & ResNet-101 &$\sim$128M& 77.32\% & 93.42\% \\ 
				\hline
				Student using LP loss (Ours)& Inception-BN &$\sim$32M& \textbf{75.91}\%& \textbf{93.13\%}\\
				\hline
				Standard back-propagation & Inception-BN &$\sim$32M& 74.80\%& 92.18\%\\
				\hline
				FitNet~\cite{romero2014fitnets} & Inception-BN &$\sim$32M& 75.52\%& 92.73\%\\
				\hline
				Knowledge Distillation~\cite{hinton2015distilling}& Inception-BN &$\sim$32M& 75.44\%& 92.65\%\\
				\hline
				Attention Transfer~\cite{paying}& Inception-BN &$\sim$32M& 75.36\%& 92.74\%\\
				\hline
				Neuron Selectivity Transfer~\cite{huang2017like}& Inception-BN &$\sim$32M& 75.66\%& 92.89\%\\
				\hline
			\end{tabular}
		\end{center}
	\end{table*}
	
	\subsection{Validations on ImageNet}
	We next conducted experiments on an extremely large image dataset, namely ImageNet ILSVRC 2012~\cite{krizhevsky2012imagenet}, which has about 1.28M training images and 50K validation images. We followed the settings in Huang and Wang~\cite{huang2017like} to implement the teacher-student learning paradigm. Wherein, we used ResNet-101~\cite{he2016deep} and Inception-BN~\cite{ioffe2015batch} as the teacher and student networks, respectively. The teacher network has about 128M parameters and student network has only about 32M parameters, which is a relatively portable model. We used the scale and aspect ratio augmentation in Ioffe and Szegedy~\cite{ioffe2015batch} and color augmentation in Krizhevsky~\etal~\cite{krizhevsky2012imagenet} following Huang and Wang~\cite{huang2017like}. 
	
	These networks were optimized using Nesterov Accelerated Gradient (NAG), and the weight decay and the momentum were set as $10^{-4}$ and 0.9, respectively. We trained the networks for 100 epochs, and the initial learning rate was set as 0.1 and divided by 10 at the 30, 60 and 90 epochs, respectively. The batch size was set as 256 and the hyper-parameters of the proposed method are the same as those in CIFAR experiments. In addition, $\lambda$ and $\tau$ in knowledge distillation were equal to $2$ and $0.5$, respectively~\cite{huang2017like}. For FitNet, attention transfer and neuron selectivity transfer, the value of $\lambda$ was set to $10^2$, $10^3$ and $5$, respectively, which refers to the settings in~\cite{huang2017like}.
	
	Table~\ref{table3} shows the classification results of student networks on the ImageNet dataset by exploiting the proposed method and state-of-the-art learning methods. The top-1 accuracy and the top-5 accuracy of the teacher network are $77.32\%$ and $93.42\%$, respectively. The student network without inheriting information (\ie the standard BP) from the teacher achieved a $74.80\%$ top-1 accuracy and a $92.18\%$ top-5 accuracy.
	
	It can be found in Table~\ref{table3} that, all the teacher-student learning methods achieve better results than that of the standard back-propagation, which demonstrates that there are abundant information in the teacher network. The student network learned using the proposed method obtained a $93.13\%$ top-5 accuracy, which is about $1\%$ higher than that of the original student network. When compared to other methods, the student network learned through the proposed locality preserving loss achieved the highest accuracy. In addition, as discussed in Proposition~\ref{prop}, the proposed method has significantly lower computational and space complexities. Especially, deep neural networks on the ImageNet dataset have sophisticated architectures with billions of parameters, and the proposed method is more suitable for efficiently learning portable student networks. For example, the number of additional parameters for connecting teacher and student networks needed by FitNet based methods is about 5,035M. In contrast, the proposed method does not need any additional parameters but can obtain better results.
	
	\subsection{Intermediate Layer Selection}
	
	Recalling Eq.~\ref{Fcn:FitNet}, most of the teacher-student learning algorithms utilize the output features of intermediate layers of the teacher and student networks (\ie the guide layer and hint layer). An important issue is how to select the hint layer from teacher network and the guided layer from the student network and how would the selection of such layers affect the performance. Zagoruyko and Komodakis~\cite{zagoruyko2016paying} has discussed this problem and the experimental results showed that different selections of intermediate layers could only have a minor influence on the accuracy of student networks (less than 0.2\% on the CIFAR-10 dataset). Moreover, using more than 1 guided layer/hint layer results in a slight improvement of performance (less than 0.1\% on the CIFAR-10 dataset). Therefore, we do not study the selection of intermediate layers. Instead, we focus on the performance of the student networks under the same experimental settings.

	\section{Conclusion}\label{sec:conclu}
	Here we examine the deep neural network compression problem for learning portable networks from original teacher models. Besides features generated by the teacher network, the relationship between samples in the feature space is another important information for maintaining the performance of the student network. In this paper, we present a novel teacher-student learning paradigm by introducing the locality preserving loss. The neighbor relationship between samples represented by the teacher network is embedded into the student network. Therefore, the resulting portable network can achieve similar performance as that of the teacher network naturally. Experiments on benchmark datasets show that the proposed method can efficiently produce portable neural networks with higher performance, which is superior to the state-of-the-art approaches.

	\ifCLASSOPTIONcaptionsoff
	\newpage
	\fi

	\bibliography{ref}
	\bibliographystyle{IEEEtran}

\end{document}